\documentclass{article}
\usepackage{arxiv}

\usepackage{cite}
\usepackage{amsmath,amssymb,amsfonts}
\usepackage{algorithmic}
\usepackage{graphicx}
\usepackage{textcomp}
\usepackage{xcolor}
    
\usepackage{graphicx}
\usepackage[utf8]{inputenc}
\usepackage{amssymb,amsmath,array}

\usepackage{microtype}      
\usepackage{subfig}
\usepackage{epstopdf}

\usepackage{bm}
\usepackage{amsthm}
\usepackage{mathrsfs}
\usepackage{nicefrac}
\usepackage{xfrac}
\usepackage{multirow}
\usepackage{booktabs}
\usepackage{arydshln}
\usepackage{epstopdf}
\usepackage[hidelinks]{hyperref}

\newlength\replength
\newcommand\repfrac{.66}

\newcommand\rulewidth{.8pt}
\newcommand\tdashfill[1][\repfrac]{\cleaders\hbox to \replength{%
  \smash{\rule[\arraystretch\ht\strutbox]{\repfrac\replength}{\rulewidth}}}\hfill}

\DeclareFontFamily{OT1}{pzc}{}
\DeclareFontShape{OT1}{pzc}{m}{it}{<-> s * [1.10] pzcmi7t}{}
\DeclareMathAlphabet{\mathpzc}{OT1}{pzc}{m}{it}
\DeclareMathOperator{\sign}{sign}

\newtheorem{definition}{Definition}[section]
\newtheorem{theorem}{Proposition}[section]

\newtheorem{lemma}[theorem]{Lemma}

\newcommand\defeq{\mathrel{\overset{\makebox[0pt]{\mbox{\normalfont\tiny def}}}{=}}}
\begin{document}

\title{Wasserstein Exponential Kernels}

\author{
 Henri De Plaen\\
  ESAT-STADIUS\\
  KU Leuven\\
  Leuven, 3001 Belgium \\
  \texttt{henri.deplaen@esat.kuleuven.be} \\
   \And
   Michaël Fanuel\\
  ESAT-STADIUS\\
  KU Leuven\\
  Leuven, 3001 Belgium \\
  \texttt{michael.fanuel@esat.kuleuven.be} \\
   \And
 Johan A. K. Suykens\\
  ESAT-STADIUS\\
  KU Leuven\\
  Leuven, 3001 Belgium \\
  \texttt{johan.suykens@esat.kuleuven.be} \\
}

\maketitle

\begin{abstract}
In the context of kernel methods, the similarity between data points is encoded by the kernel function which is often defined thanks to the Euclidean distance, a common example being the squared exponential kernel. Recently, other distances relying on optimal transport theory -- such as the Wasserstein distance between probability distributions -- have shown their practical relevance for different machine learning techniques. In this paper, we study the use of exponential kernels defined thanks to the regularized Wasserstein distance and discuss their positive definiteness. More specifically, we define Wasserstein feature maps and illustrate their interest for supervised learning problems involving shapes and images. Empirically, Wasserstein squared exponential kernels are shown to yield smaller classification errors on small training sets of shapes, compared to analogous classifiers using Euclidean distances.
\end{abstract}

\section{Introduction}
Contemporary machine learning methods frequently rely on neural networks, and shape recognition relies more specifically on convolutional neural networks. The big advantage of the latter is its ability to take the underlying structure of the data into account by treating neighboring pixels together. If these methods are very often impressive by their performance, they are also known for their drawbacks such as a weak robustness and a difficult explainability. On the other side, though not always being as accurate as neural networks, kernel methods are praised for their easy explainability and robustness. Another advantage of kernel methods is their versatility as they easily be used in supervised and unsupervised methods, as well as for generation~\cite{GenRKM}. We emphasize here the interest of choosing a particular kernel based on Wasserstein distance for classifying small datasets consisting of shapes.

In the context of kernel methods, squared exponential kernel functions are widely used, mainly because of their universal approximation properties and their empirical success. These Gaussians consist of the exponential of the negative Euclidean distance squared. However, the Euclidean distance might not always be appropriate to compare data points when data has some specific structure. Indeed, it measures the correspondence of each feature independently of the other features. For example, let's consider the case of two identical 2D-shapes. When the two shapes overlap, their Euclidean distance is zero. However, if they do not overlap, their relative Euclidean distance becomes large although the shapes are identical. In other words, the Euclidean distance only compares each pixel at the same place on the grid, not taking the neighbouring pixels into account. The general structure of the features is not taken into account, only their strict correspondence.
However, another distance -- the Wasserstein distance -- gained popularity in  recent years since it can incorporate the structure of the data if the dataset can be processed so that the datapoints can be considered as probability distributions.
\subsection*{Contributions}
The contributions of this paper are the following. Empirically, we demonstrate that squared exponential kernels~\eqref{eq:kW} based on a regularized Wassertein distance  are performant on small scale classification problems involving shape datasets, compared for instance to the popular Gaussian RBF kernel~\cite{rasmussen:williams:2006}. Also, an approximation technique is proposed, with the so-called Wasserstein feature map, so that a positive semi-definite (psd) kernel can be defined from the Wasserstein squared exponential kernel which is not necessarily psd.
\subsection*{Notations and conventions}
In the sequel, we denote vectors by bold lower case letters. Let $\bm{1}$ be the all ones column vector. Also, we define $\delta_y$  to be the Dirac measure at point $y$. A kernel $k:\mathbb{R}^d\times \mathbb{R}^d \to \mathbb{R}$ is called positive semi-definite if all kernel matrices  $K = [k(\bm{x_i},\bm{x_j})]_{i,j =1}^n$ are positive semi-definite.

\subsection*{Wasserstein distances}
The Wasserstein distance is a central notion in optimal transport theory. Also known as the \emph{earth mover's distance}, it corresponds to the optimal transportation cost between two measures~\cite{villani2008optimal,GabrielPeyre2019COTW}. Let $p>0$. We then define two normalized empirical measures $\bm{\alpha} = \sum_{i=1}^m a_i \delta_{\bm{y_i}}$ and $\bm{\beta} = \sum_{j=1}^n b_j \delta_{\bm{z_j}}$ such that $\bm{a}^\top \bm{1} = 1$ and $\bm{b}^\top \bm{1} = 1$, and where  $\{\bm{y_i}\in \mathbb{R}^d\}_{i=1}^n$, $\{\bm{z_j}\in \mathbb{R}^d\}_{j=1}^m$ are support points.
Also, we define an Euclidean distance matrix $d_{ij} = \|\bm{y_i}-\bm{z_j}\|_2$. Then, the $p$-Wasserstein distance is given by 
\begin{equation*}
\mathcal{W}_p(\bm{\alpha},\bm{\beta}) = \left(\min_{\bm{\pi} \in \Pi\left(\bm{\alpha},\bm{\beta}\right)} \sum_{i,j} \pi_{ij}d_{ij}^p\right)^{\nicefrac{1}{p}},
\end{equation*}
 with $\Pi(\bm{\alpha},\bm{\beta}) = \left\{ \Pi \in \mathbb{R}^{m \times n} | \Pi \bm{1}= \bm{a} \; \mathrm{and} \; \Pi^\top \bm{1}=\bm{b} \right\}$, the set of joint distributions $\pi$ with specified marginals given by $\alpha$ and $\beta$. Intuitively, the optimal probability distribution $\pi^\star$ represents the optimal mass transportation scheme from $\alpha$ to $\beta$. A particular result occurs in the one-dimensional ($d=1$) case assuming the support points are ordered, \emph{i.e.}, $y_1 \leq \ldots \leq y_m$ and $z_1 \leq \ldots \leq z_n$, where the Wasserstein distance reduces to an $\ell^p$-norm: $\mathcal{W}_p^p\left(\frac{1}{n} \sum_{i=1}^n\delta_{y_i}, \frac{1}{n}\sum_{j=1}^n \delta_{z_j}\right) = \tfrac{1}{n}||\bm{y}-\bm{z}||_p^p$~\cite{GabrielPeyre2019COTW}. This connection between $\ell^p$-norms and Wasserstein distances is only clear in one dimension, illustrating here again the fact that $\ell^p$-norms don't take the underlying structure into account. To take it into account, we need to consider the case $d>1$. In this way we can define a new kernel function 
 \begin{equation}
 k_W(\bm{\alpha},\bm{\beta}) = \mathrm{exp}\left(-\frac{W_2^2(\bm{\alpha},\bm{\beta})}{2\sigma^2}\right).\label{eq:kW}
 \end{equation}
 However, this has some undesirable consequences concerning positive definiteness.
A kernel $k(\bm{x},\bm{y})=\exp\left( - t f(\bm{x},\bm{y})\right)$ is positive semi-definite for all $t>0$ if and only if $f(\bm{x},\bm{y})$ is Hermitian and \emph{conditionally} negative semi-definite~\cite{Berg1984}. Recall that a kernel is \emph{conditionally} negative semi-definite if any Gram matrix $F = [f(\bm{x_i},\bm{x_j})]_{i,j =1}^n$ (with $n\geq 2)$ built from a discrete sample satisfies $\bm{c}^\top F \bm{c}\leq 0$ for all $\bm{c}$ such that $\bm{1}^\top \bm{c} =0$. However, the Wasserstein distance for $d>1$ is not necessarily \emph{conditionally} negative definite~\cite{GabrielPeyre2019COTW}. The consequence is that we cannot guarantee that any resulting squared exponential kernel matrix built with the 2-Wasserstein distance is positive definite. This property is fundamental in kernel theory and more specifically for defining \emph{reproducing kernel Hilbert spaces} (RKHS; see~\cite{Scholkopf:2001:LKS:559923} for more details).


\section{Dealing with indefinite exponential kernels}
This restriction has lead authors to consider only some specific cases of Wasserstein distances which are known to be positive definite. The one-dimensional generic case is proven to be positive definite and has lead to the introduction of sliced Wasserstein distances~\cite{Carriere2017,SlicedWasserstein}. Another notable case is the Wasserstein distance between two Gaussians in more than one dimension, which can even be written in closed form~\cite{GabrielPeyre2019COTW}.

Some kernel methods are still usable with non positive definite kernels, such as LS-SVMs \cite{suykens:worldsci2002,Huang2017}. However, this leads to a slightly different interpretation of the global problem, using Kre\u{\i}n  spaces for which a weaker version of the representer theorem holds~\cite{Ong2004}. In this paper, we propose an alternative which allows us to still work with a positive definite kernel approximating the squared exponential kernel. If the Wasserstein exponential kernel can not be used, we can always find a parameter $\sigma>0$ and a finite dimensional feature map resulting in a positive definite kernel.


\subsection{Positive definite squared exponential kernels and bandwidth choice}
In this section, we show that for a given dataset, the corresponding Gram matrix of $k_W$ is positive definite if the bandwidth parameter $\sigma>0$ is small enough. 
\begin{definition}\label{def1}
Let $d: \mathcal{D} \times \mathcal{D} \rightarrow \mathbb{R}_{\geq 0}$ be a symmetric function such that $d(\bm{x},\bm{x}) =0$ and let $\left\{\bm{x}_i \in \mathcal{D} \right\}_{i=1}^N$ be a dataset. A squared exponential kernel matrix is defined as
\begin{equation*}
    \bm{K}_{d,\sigma} = \left[\exp\left( \dfrac{-d^2(\bm{x}_i,\bm{x}_j)}{2\sigma^2} \right) \right]_{i,j=1}^N.
\end{equation*}
\end{definition}
By construction, this exponential kernel matrix will be symmetric and have a diagonal consisting only of ones. Its eigenvalues are real. To investigate its (semi)-definiteness, we have to investigate the sign of the minimum eigenvalue.
The minimum eigenvalue $\lambda_{\mathrm{min}}\left( \sigma \right)$ of $\bm{K}_{d,\sigma}$ is the function $\lambda_{\mathrm{min}} : \mathbb{R}_{>0} \rightarrow \mathbb{R}, \sigma \mapsto \min \left\{ \lambda_1, \ldots , \lambda_N \right\}$ where $\lambda_1, \ldots , \lambda_N$ are the eigenvalues of $\bm{K}_{d,\sigma}$.
We can now prove the following result:
\begin{lemma}
\label{lemma1}
The eigenvalues of the exponential kernel matrix $\bm{K}_{d,\sigma}$ are continuous functions of $\sigma$. In particular, $\lambda_{\mathrm{min}}\left( \sigma \right)$ is continuous.
\end{lemma}
\begin{proof}
This is a direct consequence of the continuity of the roots of a polynomial under continuously varying coefficients. Therefore, we have to prove that the coefficients of the characteristic polynomial of the exponential kernel matrix $\bm{K}_{d,\sigma}$ is continuous in function of $\sigma$. The characteristic polynomial is given by $\mathrm{det}\left(\bm{K}_{d,\sigma}-\lambda \bm{I} \right)$ and by the formula of Leibniz, we ultimately have that the characteristic polynomial is a sum of products of elements of $\bm{K}_{d,\sigma}-\lambda \bm{I}$, which are continuous in function of $\sigma$. Hence, the coefficients are continuous and so are the eigenvalues.
\end{proof}

\begin{lemma}
\label{lemma2}
$\lim_{\sigma \to 0} \bm{K}_{d,\sigma} = \mathrm{id}$ and thus $\lambda_{\mathrm{min}}\left(0\right) = 1$.
\end{lemma}
\begin{proof}
From Definition~\ref{def1}, we know that $\left[ \bm{K}_{d,\sigma} \right]_{i,j} = \exp\left( \tfrac{-d^2(\bm{x}_i,\bm{x}_j)}{2\sigma^2} \right)$ with $d^2(\bm{x}_i,\bm{x}_i)=0$ and $d^2(\bm{x}_i,\bm{x}_j)>0$ for $i\neq j$.  Denote $C_{i,j} = d^2(\bm{x}_i,\bm{x}_j)$ for simplicity. We have thus $\lim_{\sigma \to 0} \exp\left( \tfrac{0}{2\sigma^2} \right) = 1$ and $\lim_{\sigma \to 0} \exp\left(- \tfrac{C_{i,j}}{2\sigma^2} \right) = 0$ with $C_{i,j}>0$ for $i\neq j$, hence the identity matrix. By consequence, all the eigenvalues are equal to 1.
\end{proof}

\begin{lemma}
\label{lemma3}
We have 
$\lim_{\sigma \to \infty} \bm{K}_{d,\sigma} = \bm{1}\bm{1}^T$ and thus $\lim_{\sigma \to \infty}\lambda_{\mathrm{min}}\left(\sigma\right) = 0$.
\end{lemma}
\begin{proof}
Similarly as before, we have $\lim_{\sigma \to +\infty} \left[ \bm{K}_{d,\sigma} \right]_{i,j} = 1$ everywhere. By consequence, we have $\lambda_{\mathrm{max}}=N$ and all others equal to zero, hence $\lambda_{\mathrm{min}}=0$.
\end{proof}

\begin{theorem}
\label{existence-sigma}
There exists a $\sigma_{\mathrm{PSD}}\in \mathbb{R}_{+}$ such that $\bm{K}_{d,\sigma}$ is positive semi-definite for all $\sigma \leq \sigma_{\mathrm{PSD}}$.
\end{theorem}
\begin{proof}
Let's proceed \emph{ad absurdum} and suppose this is not the case. We consider the sequence $\left( \sigma_n \right)_n$ converging to 0 with $\sigma_0 = \sigma_{\mathrm{PSD}}$. There must exist some subsequence $\left( \sigma_{n_j} \right)_j$ such that $\left( \lambda_{\mathrm{min}}\left(\sigma_{n_j}\right) \right)_j < 0$. If this sequence if finite, then is suffices to consider a new sequence with $\sigma_{\mathrm{PSD}} = \sigma_{n_{j_{\mathrm{max}}}+1}$. If this subsequence is infinite, then $\left( \lambda_{\mathrm{min}}\left(\sigma_{n}\right) \right)_n$ cannot converge to 1. This is impossible because of the continuity of $\lambda_{\mathrm{min}}\left( \sigma \right)$ (lemma~\ref{lemma1}) and its convergence to 1 (lemma~\ref{lemma2}). Hence, there exist 
some $\sigma_{\mathrm{PSD}}>0$ such that $\lambda_{\mathrm{min}}\left( \sigma \right)\geq 0$ for all $\sigma \leq \sigma_{\mathrm{PSD}}$. This proves our proposition.
\end{proof}
We can empirically see the result of Proposition~\ref{existence-sigma} in Fig.~\ref{kernels}, where all eigenvalues are positive. To give some intuition, decreasing the $\sigma$ tends to make the smallest distances more predominant, pushing the smallest eigenvalue progressively to the positive side. In this sense, an indefinite kernel matrix with $\sigma$ close to $\sigma_{\mathrm{PSD}}$ will lead to very proportionally very small negative eigenvalues in magnitude. In this case, a finite positive definite approximation can be justified.
\begin{figure}[]
    \centering
    \subfloat[$\ell^2$]{\includegraphics[width=0.35\textwidth]{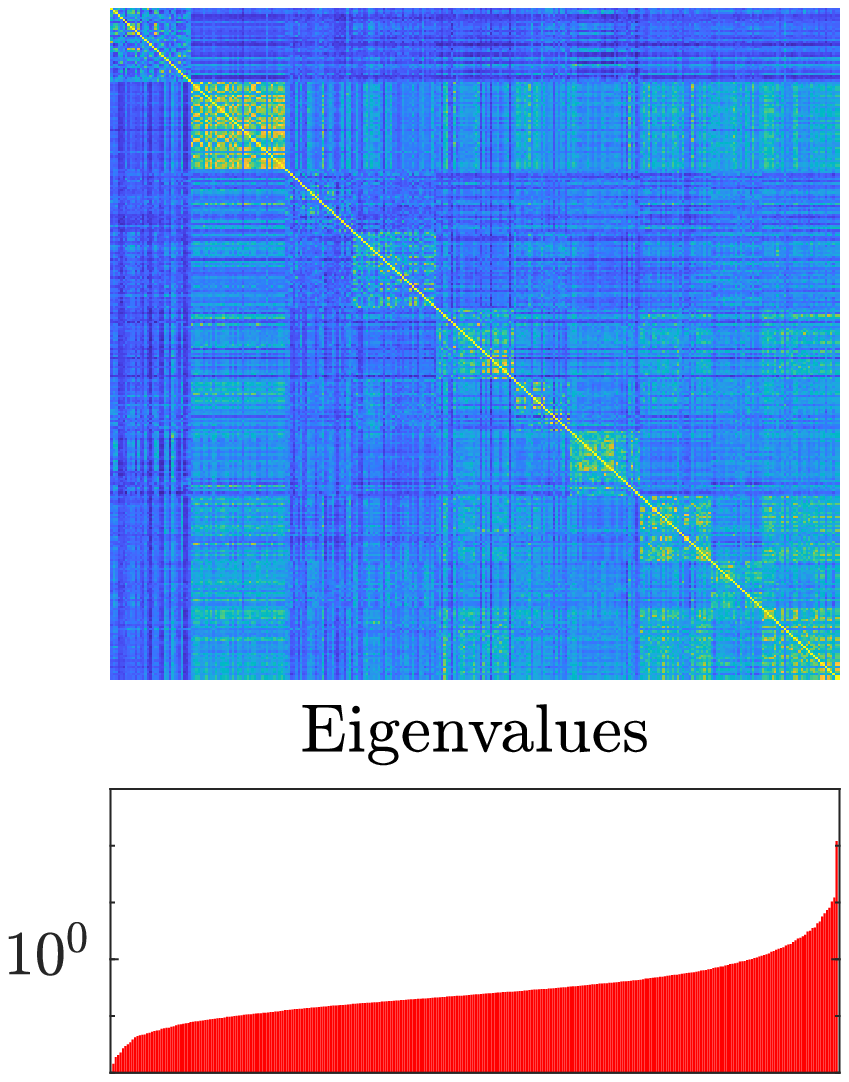}}
    \subfloat[$\mathcal{W}_2^2$]{\includegraphics[width=0.35\textwidth]{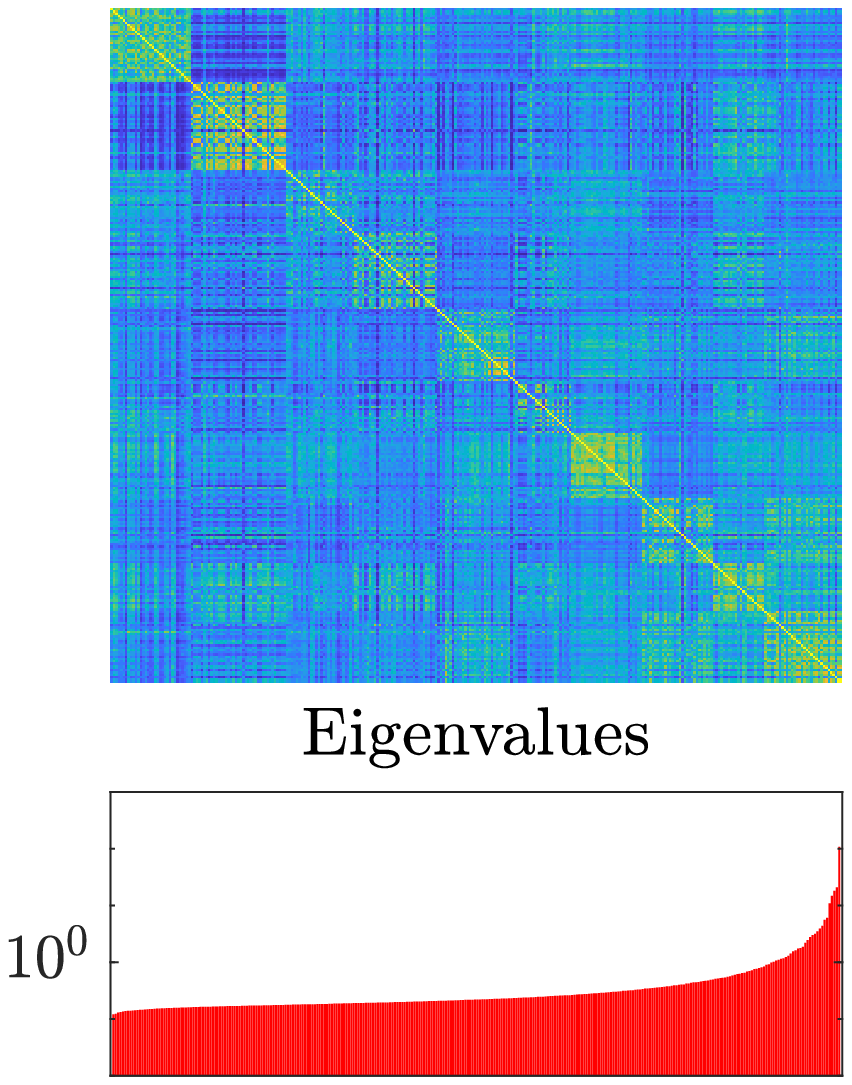}}
    \caption{Comparison of the classical squared exponential kernel matrix (based on a $\ell^2$-distance) and the introduced Wasserstein exponential kernel matrix on 250 normalized digits of the MNIST dataset~\cite{lecun-mnisthandwrittendigit-2010}. The digits are ordered by class in ascending order.}
    \label{kernels}
\end{figure}

\subsection{Wasserstein features}
We can consider a finite dimensional feature map $\bm{\phi}(\bm{x})$ such that the positive semi-definite kernel $\bm{\phi}(\bm{x})^\top \bm{\phi}(\bm{y})$ approximates $k_W(\bm{x},\bm{y})$ given in~\eqref{eq:kW}. This finite approximation is based on a training dataset $\left\{ \bm{x}_i\right\}_{i=1}^N$ for constructing an original kernel matrix $\bm{K} = \left[k_W(\bm{x}_i,\bm{x}_j)\right]_{i,j=1}^N \in \mathbb{R}^{N \times N}$. It suffices to truncate the spectral decomposition of the kernel matrix $\bm{K} = \sum_{l=1}^N \lambda_l \bm{v}_l \bm{v}_l^\top$ to the $\ell$ largest strictly positive eigenvalues. This will result in a new positive definite kernel matrix $\bm{K}^{(\ell)} \defeq \sum_{l=1}^\ell \lambda_l \bm{v}_l \bm{v}_l^\top \succ 0$ with $\lambda_1\geq \dots\geq \lambda_N$. We can now reconstruct the different components of an approximate feature map
\begin{equation}
    \phi_l\left(\bm{x} \right) \defeq \tfrac{1}{\sqrt{\lambda_l}} \bm{k}_{\bm{x}}^\top \bm{v}_l, \qquad \mathrm{for}\; i=1,\ldots,\ell,\label{eq:ApproxFeatureMap}
\end{equation}
with $\bm{k}_{\bm{x}} \defeq \left[ k_W(\bm{x},\bm{x}_1) \; \cdots \; k_W(\bm{x},\bm{x}_N) \right]^\top$. We refer to these different components as the \emph{Wasserstein features} as they compose the approximate feature map $\bm{\phi}\left(\bm{x}\right) \defeq \left[ \phi_1(\bm{x}) \; \cdots \; \phi_\ell(\bm{x}) \right]^\top$ of the Wasserstein exponential kernel. This approximate feature map is constructed by using a training dataset, but can afterwards be evaluated at any out-of-sample point. By construction, we can verify that the Wassertein features evaluated on the training dataset result in the truncated kernel matrix:
\begin{theorem}We have
$    \left[ \bm{\phi}(\bm{x}_i)^\top \bm{\phi}(\bm{x}_j)\right]_{i,j=1}^N = \bm{K}^{(\ell)}.
$\label{prop:Kk}
\end{theorem}
\begin{proof}
It suffices to observe that $\bm{k}_{\bm{x}_i} = \sum_{l=1}^N \lambda_l \bm{v}_l \left[\bm{v}_l\right]_i$. By consequence, we have $\phi_l(\bm{x}_i) = \sqrt{\lambda_l}\left[\bm{v}_l \right]_i$.
\end{proof}
Proposition~\ref{existence-sigma} suggests that even if no suitable $\sigma$ can be found such that the kernel matrix is positive, the negative eigenvalues will remain very small in magnitude. By consequence, we can suppress them without loosing much information. A truncated kernel is thus very close to the original one in spectral norm. This justifies the Wasserstein features in this sense that they are very close to the Wasserstein exponential kernel as well as being positive definite by construction. This fact can be visualized on Fig.~\ref{wass-features}.

\begin{figure*}[]
    \centering
    \subfloat[$\ell =5$]{\includegraphics[width=0.23\textwidth]{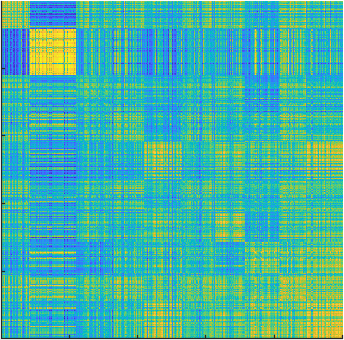}}\hfill
    \subfloat[$\ell=15$]{\includegraphics[width=0.23\textwidth]{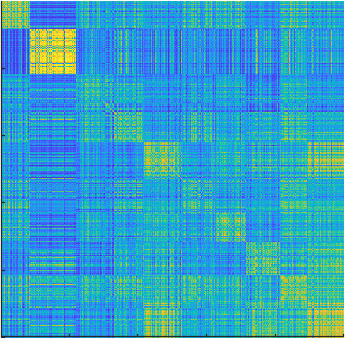}}\hfill
    \subfloat[$\ell=250$]{\includegraphics[width=0.23\textwidth]{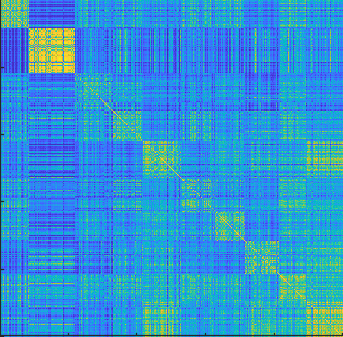}}\hfill
    \subfloat[Exact]{\includegraphics[width=0.23\textwidth]{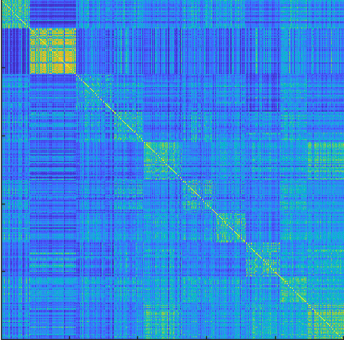}}
    \caption{Kernel matrices constructed as the inner products of a different number Wasserstein features of a test set. These matrices are compared with the exact Wasserstein squared exponential kernel matrix of the test set. Both the training set and the test are of size $N=500$.}
    \label{wass-features}
\end{figure*}

Clearly, the \emph{Wasserstein features} yield a positive semi-definite kernel. Moreover, it is also advantageous to work with finite dimensional feature maps to reduce the training time. Indeed, the computation of the Wasserstein distance (or an approximation with \emph{e.g.} Sinkhorn's algorithm~\cite{Lightspeed}) is still relatively expensive compared to $\ell^2$ distance. 

\section{Experiments}

\begin{table*}[] 
\begin{center}
\caption{Percentage of classification error on the test set of three datasets. The standard deviation is given in parenthesis. The number of repeated simulations is 7 for MNIST, 8 for Quickdraw and 6 for USPS.}
   \label{table:results}
    \begin{tabular}{r c c c c c c c}
      \textbf{Dataset} & \multicolumn{2}{c}{\textbf{MNIST}} & \multicolumn{2}{c}{\textbf{Quickdraw}} & \multicolumn{2}{c}{\textbf{USPS }}   \\
      \cmidrule(lr){2-3}\cmidrule(lr){4-5}\cmidrule(lr){6-7} \\
      \textbf{Method} &\textbf{Avg.} & \textbf{Best} & \textbf{Avg.} & \textbf{Best} & \textbf{Avg.} & \textbf{Best} \\
      Wass. LS-SVM (Core+OOS) & 3.95 ($\pm$ 0.18) & 3.74 & 11.45 ($\pm$ 0.39) & 10.97 & 6.77 ($\pm$ 0.52) & 6.20\\ 
      Wass. LS-SVM (Core) & 3.81 ($\pm$ 0.34) & 3.28 & 10.80 ($\pm$ 0.19) & 10.52 & 7.93 ($\pm$ 1.45) & 6.35\\
      Wass. LS-SVM (Indef.) & \textbf{3.40} ($\pm$ 0.11) & \textbf{3.23} & \textbf{10.75} ($\pm$ 0.27) & 10.35 & 6.15 ($\pm$ 0.67) & 5.45\\
      R. Wass. LS-SVM (Core+OOS) & 3.91 ($\pm$ 0.27) & 3.45 & 11.79 ($\pm$ 0.48) & 10.95 & 6.68 ($\pm$ 0.80) & 5.70\\ 
      R. Wass. LS-SVM (Core) & 3.71 ($\pm$ 0.15) & 3.46 & 10.99 ($\pm$ 0.44) & \textbf{10.07} & 6.35 ($\pm$ 0.11) & 6.20\\
      R. Wass. LS-SVM (Indef.) & 3.48 ($\pm$ 0.13) & 3.29 & 12.43 ($\pm$ 0.43) & 11.95 & \textbf{5.70} ($\pm$ 0.29) & \textbf{5.40}\\
      Wass. kNN & 6.31 ($\pm$ 0.33) & 5.81 & 12.26 ($\pm$ 0.33) & 11.91 & 6.60 ($\pm$ 0.44) & 6.00\\
      & & & & & & \\
      RBF LS-SVM & 4.26 ($\pm$ 0.10) & 4.07 & 11.46 ($\pm$ 0.20) & 11.23 & 6.75 ($\pm$ 0.04) & 6.70 \\
      $\ell^2$ kNN & 7.20 ($\pm$ 0.15) & 6.95 & 15.32 ($\pm$ 0.40) & 14.68 & 7.52 ($\pm$ 0.38) & 7.20 \\
      \cmidrule(lr){2-3}\cmidrule(lr){4-5}\cmidrule(lr){6-7}\\
      \textbf{Set size} &\textbf{Core + OOS} & \textbf{Others} & \textbf{Core + OOS} & \textbf{Others} & \textbf{Core + OOS} & \textbf{Others} \\
      Training & 1500 $+$ 2500 & 4000 & 500 + 750 & 1250 & 1000 + 1500 & 2500 \\
      Validation & 5000 & 5000 & 5000 & 5000 & 2000 & 2000 \\
      Test & 10 000 & 10 000 & 10 000 & 10 000 & 2000 & 2000
    \end{tabular}
  \end{center}
\end{table*}

\begin{figure*}[]
    \centering
    \subfloat[Comparison of Wasserstein exponential kernels with other similar methods.]{\includegraphics[width=0.49\textwidth]{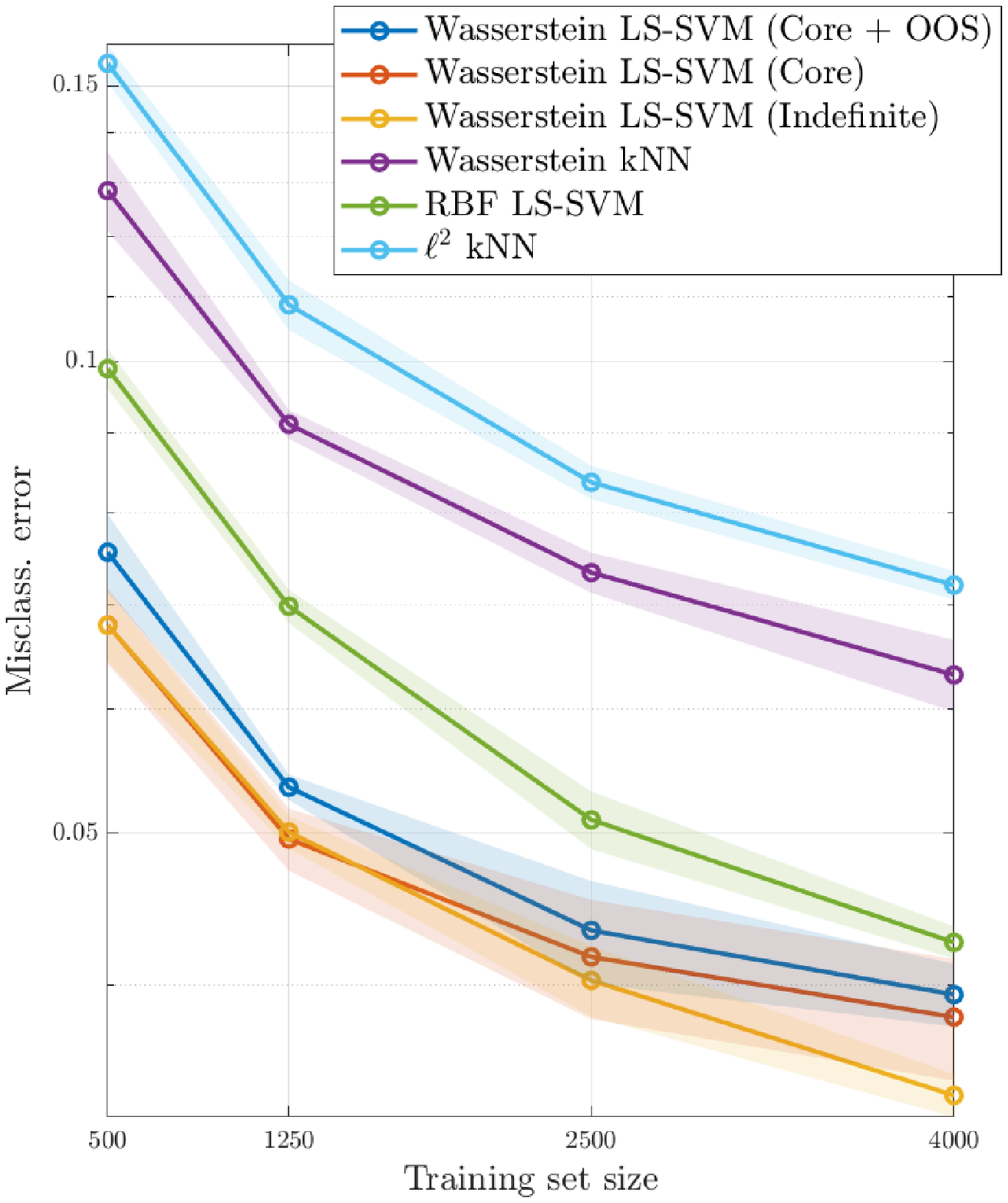}}\hfill
    \subfloat[Comparison of \emph{reweighted} Wasserstein exponential kernels with other similar methods.]{\includegraphics[width=0.49\textwidth]{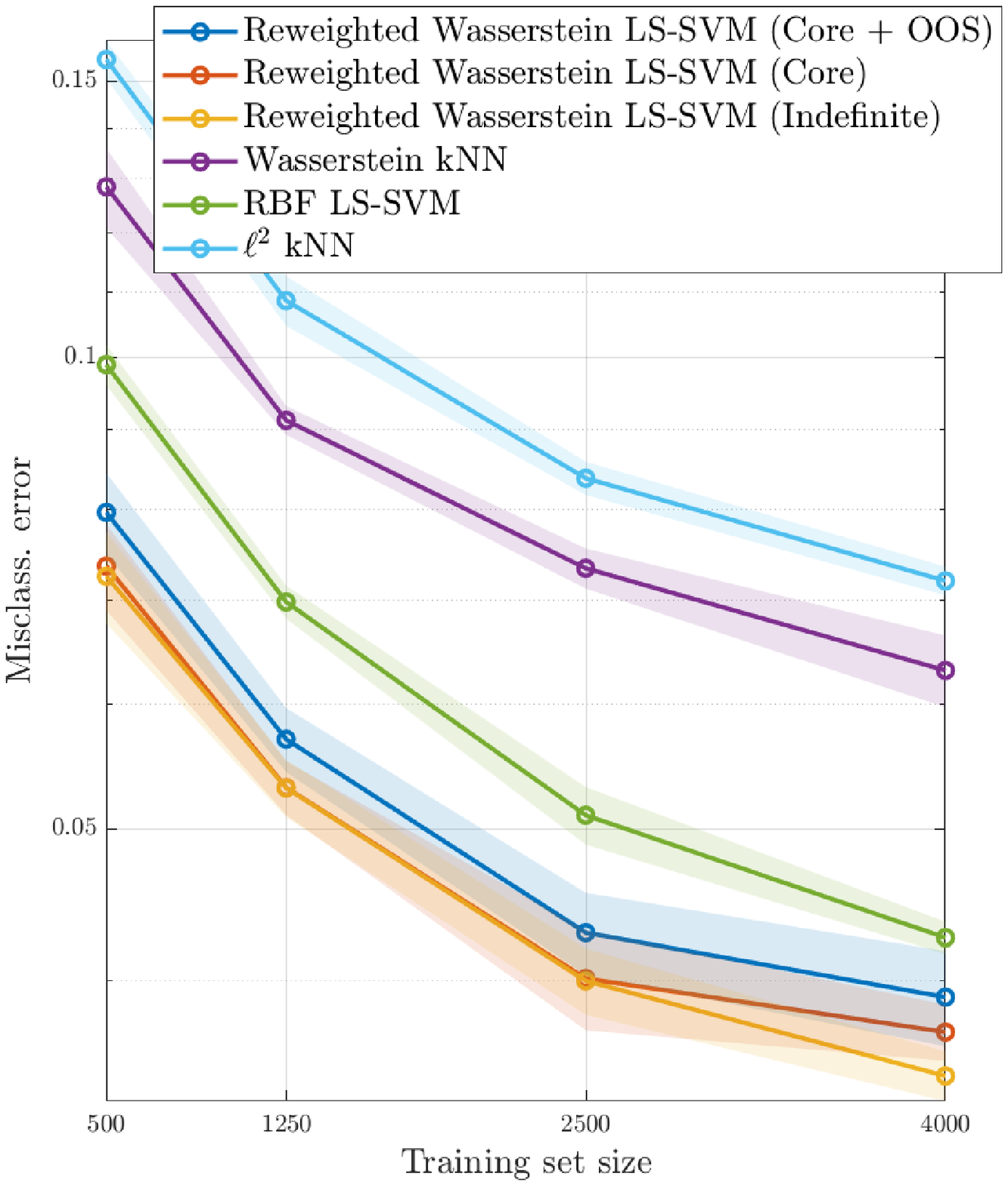}}
    \caption{ Mean misclassification rates for various subset sizes of the MNIST dataset, computed on 7 simulations. The standard deviation is given by the errors bars. For the specific case of ``Core + OOS", the out-of-sample subset represents 300 datapoints on 500, 750 on 1250, 1500 on 2500 and 2500 on 4000. The size of validation set is always 5000 and of the test set always 10 000.}
    \label{expe-shape}
\end{figure*}

\subsection{Setup for 2D shape classification} 
Let $\bm{u}$ be a greyscale image that we unfold as a vector of length $m$ and so that $u_i>0$ is the ``grey" value at the pixel $\bm{y_i}$ of a pixel grid. It is mapped to a probability $\bm{\alpha} = \sum_{i=1}^m a_i \delta_{\bm{y_i}}$ by defining $a_i = u_i/\|\bm{u}\|_1$, so that the mass of $\bm{\alpha}$ is one.
In practice, the $p=2$ Wassertein distance is calculated in this paper with the help of the well-known entropic regularization, namely
\begin{equation*}
\mathcal{W}_2^{2}(\bm{\alpha},\bm{\beta},\epsilon) = \min_{\bm{\pi} \in \Pi\left(\bm{\alpha},\bm{\beta}\right)} \sum_{i,j} \pi_{ij}d_{ij}^2 +\epsilon \pi_{ij}\log \pi_{ij},
\end{equation*}
where $\epsilon>0$ is a small regularization term and $d_{ij}$ is the Euclidean distance between pixels located at $\bm{y_i}$ and $\bm{y_j}$ in a pixel grid. The advantage of this regularized problem is that its solution can be efficiently obtained thanks to the Sinkhorn algorithm, which can be parallelized. For more details, we refer to~\cite{GabrielPeyre2019COTW}. All the simulations used $\epsilon = 0.4$ and the diagonal of the distance matrix set to zero.

\subsection{Shape recognition}
We illustrate the use of the Wasserstein based kernels in the context of shape classifications. Namely, we train a Least Squares Support Vector Machine~\cite{LSSVM} classifier on  subsets of the MNIST~\cite{lecun-mnisthandwrittendigit-2010}, Quickdraw\footnote{\texttt{https://quickdraw.withgoogle.com/data}} and USPS~\cite{USPS} datasets, which are sampled uniformly at random.
These three datasets contain handwritten digits and shapes.
The multiclass problem is solved by a one-versus-one encoding. One instance of these binary classifiers $f(\bm{x}) = \sign( \bm{w}^{\star\top}\bm{\phi}(\bm{x}) +b^\star)$ is obtained by solving
\begin{equation}
\min_{\substack{\bm{w}\in \mathbb{R}^\ell;b\in \mathbb{R}\\
e_i\in \mathbb{R}}}\bm{w}^\top\bm{w} + \frac{\gamma}{N}\sum_{i=1}^{N} e_i^2 \text{ s.t. } e_i = y_i-\bm{w}^\top\bm{\phi}(\bm{x}_i) -b, \label{eq:primal}
\end{equation}
where $y_i\in \{-1,1\}$ and $\bm{\phi}(\bm{x})\in \mathbb{R}^\ell$ is a feature map obtained for instance thanks to~\eqref{eq:ApproxFeatureMap}.
The solution is obtained by solving
\begin{equation*}
\begin{bmatrix}
        \sum_{i}\bm{\phi}(\bm{x_i})\bm{\phi}(\bm{x_i})^\top + \frac{N}{\gamma}\mathbb{I} & \sum_{i}\bm{\phi}(\bm{x_i})\\
    \sum_{i}\bm{\phi}(\bm{x_i})^\top & N
\end{bmatrix}
\begin{bmatrix}
\bm{w}\\
b
\end{bmatrix}
 = \begin{bmatrix}
\sum_{i}y_i\bm{\phi}(\bm{x_i})\\
\sum_{i}y_i
\end{bmatrix},
\end{equation*}
which is a $(\ell+1)\times (\ell+1)$ linear system.
A classifier can also be obtained by solving the dual problem of~\eqref{eq:primal}. The optimality conditions of this dual problem yield the following $(N+1)\times (N+1)$ linear system
\begin{equation}
\begin{bmatrix}
        K + \frac{N}{\gamma}\mathbb{I} & \bm{1}\\
    \bm{1}^\top & 0
\end{bmatrix}
\begin{bmatrix}
\bm{\alpha}\\
b
\end{bmatrix}
 = \begin{bmatrix}
\bm{y}\\
0
\end{bmatrix}.\label{eq:dual}
\end{equation}
The resulting classifier has then the expression $f(\bm{x}) = \sign(\sum_{i=1}^N \alpha_i^\star k(\bm{x},\bm{x_i}) +b^\star)$.
The hyperparameters $\sigma>0$ and $\gamma>0$ are chosen by validation. The final classification is done by minimizing the hammming distance on the one-versus-one outputs~\cite{mlssvm}.
In order to account for the amount of ink  in the grey images $\bm{u}$ and $\bm{v}$, we also introduce a  reweighted kernel that is defined as 
\begin{equation}
k_{RW}(\bm{u},\bm{v}) = \|\bm{u}\|_1 \|\bm{v}\|_1 k_W\left(\frac{\bm{u}}{\|\bm{u}\|_1},\frac{\bm{v}}{\|\bm{v}\|_1}\right).\label{eq:kernelRescaled}
\end{equation} Notice that a similar kernel has been defined with the Euclidean distance in~\cite{Mairal2016,Mairal}.

In our experiments, we compare several methods based on $k_W$ and $k_{RW}$, Wasserstein and Euclidean distances.
\subsubsection{Core Wasserstein kernel}
The ``Core" method consists in solving~\eqref{eq:primal} thanks to the feature map~\eqref{eq:ApproxFeatureMap} associated to $K^{(\ell)}$. The parameter $\ell$ is chosen such that all the selected eigenvalues are larger than $10^{-6}$ to avoid numerical instabilities. The optimal $\bm{w}^\star$ and $b^\star$ are then obtained by solving a linear system.

\subsubsection{Core Wasserstein kernel with out-of-sample}
Our second method named ``Core + OOS" uses almost the same methodology as ``Core". However, a subset of the training set is used to construct the truncated Wasserstein kernel of Proposition~\ref{prop:Kk}. Then the out-of-sample (OOS) formula~\eqref{eq:ApproxFeatureMap} is used to construct an approximation of the kernel matrix on the full training dataset. The advantage of this approximation is that it can avoid the full eigendecomposition of the kernel matrix which is necessary for the ``Core" method.
\subsubsection{Indefinite Wasserstein kernel}
 For this second method, we simply use for the kernel matrix the indefinite Gram matrix associated to~\eqref{eq:kernelRescaled} and solve the system~\eqref{eq:dual} associated to the dual formulation of LS-SVM. While the associated optimization problem is not necessarily bounded in that case, the linear system~\eqref{eq:dual} still has a solution in practice (almost surely if $\gamma$ selected uniformly at random.). We name this method ``Indefinite Wasserstein" in Figure~\ref{expe-shape}.
\subsubsection{Gaussian RBF}
The previous methods are compared with a classical LS-SVM classifier with kernel
\[
k(\bm{u},\bm{v}) = \exp\left( \dfrac{-\|\bm{u}-\bm{v}\|_2^2}{2\sigma^2} \right).
\]
The parameters $\sigma$ and $\gamma$ are obtained by validation in the same spirit as above.
\subsubsection{KNN}
The same task is also performed for a kNN classifiers defined both with Euclidean and Wasserstein distances~\cite{KNN_OT}. Those two methods are considered as benchmarks to assess the accuracy of the kernel methods hereabove. Notice that the number of nearest neighbours $k$ is selected by validation.
\subsection{Description of the simulations}
The simulations are repeated several times and the mean classification error rate is given as well as the standard deviation. We emphasize that the classes are balanced in each of the datasets. The coded is provided on GitHub\footnote{\texttt{https://github.com/hdeplaen/Exponential\_Wasserstein\_Kernels}}.

\subsection{Discussion}
The results obtained by classifiers defined with Wasserstein exponential kernel $k_W$ outperform the Euclidean and Wasserstein kNN classifiers, as well as LS-SVM with a Gaussian RBF kernel (see Fig.~\ref{expe-shape} and Table~\ref{table:results}). The latter is especially outperformed when the number of training data points is limited to a few thousands. We observed empirically that the advantage of $k_W$ is indeed reduced as the size of the training set further increases. Surprisingly, the classifier obtained for the indefinite $k_W$ kernel yields the best performance when the training set is larger. For moderate size training sets, LS-SVM classifiers can be competitive with respect to other methods that do not rely on convolutional neural networks. The latter are known to be performant for relatively large training datasets.
While an advantage of Wasserstein based methods is an increased accuracy in the classification tasks of this paper, a main disadvantage is the increased training time.

\section{Conclusion}
In this paper, we proposed the use of Wasserstein squared exponential kernels for classifying shapes given relatively small training datasets. Although the computation of Wasserstein distances is expensive, it can be made possible thanks to the entropic regularization and the Sinkhorn algorithm, as it is well known. The so-called Wasserstein features are also proposed to serve as an approximation of the Wasserstein squared exponential kernel which is not necessarily positive semidefinite. In particular, this construction is possible if the bandwidth parameter is small enough as it is explained by elementary theoretical results. These theoretical results also open a door to more general exponential kernels based on any measure of similarity.

\section*{Acknowledgment}
\footnotesize{
EU: The research leading to these results has received funding from the European Research Council under the European Union's Horizon 2020 research and innovation program / ERC Advanced Grant E-DUALITY (787960). This paper reflects only the authors' views and the Union is not liable for any use that may be made of the contained information. Research Council KUL: Optimization frameworks for deep kernel machines C14/18/068. Flemish Government: FWO: projects: GOA4917N (Deep Restricted Kernel Machines: Methods and Foundations), PhD/Postdoc grant. This research received funding from the Flemish Government under the “Onderzoeksprogramma Artificiële Intelligentie (AI) Vlaanderen” programme. Ford KU Leuven Research Alliance Project KUL0076 (Stability analysis and performance improvement of deep reinforcement learning algorithms). The computational resources and services used in this work were provided by the VSC (Flemish Supercomputer Center), funded by the Research Foundation - Flanders (FWO) and the Flemish Government – department EWI.}

\begin{footnotesize}


\bibliographystyle{unsrt}
\bibliography{bib}

\end{footnotesize}

\end{document}